\documentclass{article}

    \PassOptionsToPackage{numbers, compress}{natbib}

    \usepackage[preprint]{score_workshop}

\usepackage{amsmath,amsfonts,bm}

\def\eqref#1{equation~\ref{#1}}

\def\1{\bm{1}}

\def\vtheta{{\bm{\theta}}}
\def\vphi{{\bm{\phi}}}

\def\vx{{\bm{x}}}

\DeclareMathAlphabet{\mathsfit}{\encodingdefault}{\sfdefault}{m}{sl}
\SetMathAlphabet{\mathsfit}{bold}{\encodingdefault}{\sfdefault}{bx}{n}

\def\gD{{\mathcal{D}}}
\def\gE{{\mathcal{E}}}

\newcommand{\E}{\mathbb{E}}

\usepackage[utf8]{inputenc} 
\usepackage[T1]{fontenc}    
\usepackage{hyperref}       
\usepackage{url}            
\usepackage{booktabs}       
\usepackage{amsfonts}       
\usepackage{nicefrac}       
\usepackage{microtype}      
\usepackage{xcolor}         

\usepackage{amsthm}
\usepackage{thmtools, thm-restate}
\usepackage{amsmath}
\usepackage{amsfonts}
\usepackage{amssymb}
\usepackage{mathtools}
\usepackage{float}
\usepackage{subfig}
\usepackage{multicol}
\usepackage{multirow}

\newtheorem{lemma}{Lemma}

\title{Why Are Conditional Generative Models Better Than Unconditional Ones?}

\author{%
Fan Bao$^1$, Chongxuan Li$^2$, Jiacheng Sun$^3$, Jun Zhu$^1$ \\
$^1$Dept. of Comp. Sci. \& Tech., Institute for AI, BNRist Center \\
$^1$Tsinghua-Bosch Joint ML Center, THBI Lab,Tsinghua University, Beijing, 100084 China \\
$^2$Gaoling School of Artificial Intelligence, Renmin University of China, \\
$^2$Beijing Key Laboratory of Big Data Management and Analysis Methods, Beijing, China \\
$^3$Huawei Noah's Ark Lab \\
\texttt{bf19@mails.tsinghua.edu.cn};\quad \texttt{chongxuanli@ruc.edu.cn};\\
\texttt{sunjiacheng1@huawei.com};\quad \texttt{dcszj@tsinghua.edu.cn} \\
}

\begin{document}

\maketitle

\begin{abstract}
Extensive empirical evidence demonstrates that conditional generative models are easier to train and perform better than unconditional ones by exploiting the labels of data. So do score-based diffusion models. In this paper, we analyze the phenomenon formally and identify that the key of conditional learning is to partition the data properly. Inspired by the analyses, we propose \emph{self-conditioned diffusion models} (SCDM), which is trained conditioned on indices clustered by the $k$-means algorithm on the features extracted by a model pre-trained in a self-supervised manner. SCDM significantly improves the unconditional model across various datasets and achieves a record-breaking FID of 3.94 on ImageNet 64x64 without labels. Besides, SCDM achieves a slightly better FID than the corresponding conditional model on CIFAR10.
\end{abstract}

\section{Introduction}

Extensive empirical evidence in prior work~\cite{li2017triple,brock2018large,dhariwal2021diffusion} demonstrates that conditional generative models are easier to train and perform better than unconditional ones by exploiting the labels of data. So do score-based diffusion models (DM). For instance, the representative work~\cite{dhariwal2021diffusion} achieves a FID of 10.94 when trained conditionally and a FID of 26.21 when trained unconditionally on ImageNet of size 256x256. 

Intuitively, the gap exists because (1) the marginal distribution induced by a conditional model is more expressive than the corresponding  unconditional model; and (2) the data distribution conditioned on a specific class has fewer modes and is easier to fit than the original data distribution. 

In this paper, we formalize the above intuition in an ideal setting where we have infinite data. It is easy to show that the marginal distribution induced by a conditional model can be viewed as a mixture of the corresponding unconditional models. Further, we derive a sufficient condition for the superiority of the conditional model, which suggests that the conditional model gains more as the conditional data distribution gets simpler. The analyses explain previous empirical findings: conditioning on class labels probably partitions the data into simpler groups according to the semantics of data.

Notably, our analyses apply to all possible conditions, not limited to class labels. Then, the very natural idea is to find a certain way to obtain  meaningful conditions in an unsupervised manner and boost the unconditional generation results. The recent advances in self-supervised learning~\cite{he2020momentum,chen2020simple} show that one can learn predictive representations without labels, which serve as an ideal tool for obtaining meaningful conditions. Specifically, we simply run a clustering algorithm (e.g., $k$-means) on the features extracted by a model pre-trained in a self-supervised manner (on the same dataset) and use the cluster indices as conditions to train a conditional model.

Although our analyses and the self-conditional approach is applicable to all types of deep generative models, we focus on score-based diffusion models in our experiments to explore the boundary of unsupervised generative modeling. Therefore, we refer to our approach as \emph{self-conditioned diffusion models} (SCDM). We systematically evaluate SCDM on several widely adopted datasets. In all settings, SCDM significantly improves the unconditional model. Notably, SCDM achieves a record-breaking FID of 3.94 on ImageNet 64x64 without labels. Besides, SCDM achieves a slightly better FID than the corresponding conditional model on CIFAR10.

\section{Why Are Conditional Generative Models Better Than
Unconditional Ones}

In this section, we present the problem formulation and our analyses.

\subsection{Problem Formulation}
\label{sec:problem}
 
Let $q(\vx, c)$ be the joint distribution of the data $\vx$ and the condition $c$ and $q(\vx) := \sum_c q(\vx, c)$. Let $p_{\vtheta, E}(\vx)$ be a model parameterized by $\vtheta \in \Theta$ and $E \in \gE$, where $\vtheta$ denotes the parameters in the backbone and $E$ is the embedding for a condition. We formalize two learning paradigms as follows. 

In \textit{unconditional learning}, $p_{\vtheta, E}(\vx)$ approximates the marginal data distribution $q(\vx)$ directly and  $E$ is a redundant embedding shared by all data. Formally, given a certain statistics divergence $\gD$ (or more loosely a divergence upper bound~\cite{song2021maximum,bao2022analytic}), unconditional learning aims to optimize
\begin{align}
\label{eq:kl_uncond}
    \min\limits_{\vtheta \in \Theta, E \in \gE} \gD(q(\vx) \| p_{\vtheta, E}(\vx)).
\end{align}

In \emph{conditional learning}, the embedding $E$ is spared to receive the signal from the condition $c$, through an embedding function $\vphi \in \Phi$. This induces a conditional model $p_{\vtheta, \vphi}(\vx|c) \coloneqq p_{\vtheta, E}(\vx)|_{E=\vphi(c)}$, which approximates the conditional data distribution $q(\vx|c)$ by tuning the backbone $\vtheta$ and the embedding function $\vphi$. Formally, conditional learning aims to optimize
\begin{align}
\label{eq:kl_cond}
    \min\limits_{\vtheta \in \Theta, \vphi \in \Phi} \E_{q(c)}\gD(q(\vx|c) \| p_{\vtheta, \vphi}(\vx|c)).
\end{align}

The conditional model applies ancestral sampling to generate samples, where a condition $c$ is firstly drawn from $q(c)$\footnote{We assume $q(c)$ is known, which is satisfied in conditional learning with labels.}, and then a data $\vx$ is drawn from $p_{\vtheta, \vphi}(\vx|c)$. Such a process produces samples from $p_{\vtheta, \vphi}(\vx) \coloneqq \E_{q(c)} p_{\vtheta, \vphi}(\vx|c)$. The generation performance of the conditional model is evaluated according to how close $p_{\vtheta, \vphi}(\vx)$ is to the data distribution $q(\vx)$, i.e., $\gD(q(\vx) \| p_{\vtheta, \vphi}(\vx))$.

\subsection{Analyses}

In this section, we attempt to formalize two insights on why conditional learning of generative models generally outperforms the unconditional one. 

Firstly, we compare the expressive power of the two strategies with the same backbone parameterized by $\vtheta$. As shown in Section~\ref{sec:problem}, the conditional model produces samples from $p_{\vtheta, \vphi}(\vx) = \E_{q(c)} [p_{\vtheta, \vphi}(\vx|c)] = \E_{q(c)} [p_{\vtheta, E} (\vx) |_{E = \vphi(c)} ]$. Therefore, $p_{\vtheta, \vphi}(\vx)$ can be viewed as a mixture of several unconditional models. Namely, the conditional model is more expressive than the unconditional one, despite the fact that both models are based on the same backbone $p_{\vtheta, E}(\vx)$. 
 
Secondly, we derive a sufficient condition for the superiority 
of the conditional model. Let $\vtheta^*_u, E^*_u$ be the optimal solution of the unconditional learning in Eq.~{(\ref{eq:kl_uncond})}. Let $\vtheta_c^*, \vphi_c^*$ be the optimal solution of the conditional learning in Eq.~{(\ref{eq:kl_cond})}. Proposition~\ref{prop:better} characterizes a sufficient condition for $\gD(q(\vx) \| p_{\vtheta_c^*, \vphi_c^*}(\vx)) < \gD(q(\vx) \| p_{\vtheta^*_u, E^*_u} (\vx))$.

\begin{restatable}{proposition}{better}
\label{prop:better}
Suppose for any parameter $\vtheta \in \Theta$ and any condition $c$, approximating $q(\vx|c)$ is simpler than $q(\vx)$ by only tuning the embedding $E$ of $p_{\vtheta, E}(\vx)$, i.e., $\min_E \gD (q(\vx|c) \| p_{\vtheta, E}(\vx)) < \min_E \gD(q(\vx) \| p_{\vtheta, E}(\vx))$.
Then, under additional mild regularity conditions~\footnote{Specifically,  we assume that the divergence $\gD$ is convex and the embedding function space $\Phi$ includes all measurable functions, which are verifiable in practice.  In fact, the former can be satisfied using the KL divergence and the latter can be satisfied by using nonparametric embeddings.}, $\gD(q(\vx) \| p_{\vtheta_c^*, \vphi_c^*}(\vx)) < \gD(q(\vx) \| p_{\vtheta^*_u, E^*_u} (\vx))$ holds. (Proof in Appendix~\ref{sec:proof})
\end{restatable}

\begin{table}[t]
    \caption{FID$\downarrow$ results on different datasets. $K$ represents the number of clusters.}
    \vspace{.1cm}
    \begin{center}
    \begin{tabular}{lrrrr}
    \toprule
         & CIFAR10 & CelebA 64x64 & LSUN Bedroom 64x64 & ImageNet 64x64 \\
         \midrule
        Unconditional DM & 2.72 & 2.14 & 2.69 & 6.44\\
        Conditional DM & 2.24 & - & - & \textbf{3.08}\\
        \midrule
        SCDM ($K=2$) & - & 2.04 & - & -\\
        SCDM ($K=10$) & \textbf{2.23} & \textbf{1.91} & - & -\\
        SCDM ($K=20$) & 2.27 & 2.08 & 2.39 & -\\
        SCDM ($K=30$) & 2.30 & - & - & -\\
        SCDM ($K=50$) & 2.34 & - & - & -\\
        SCDM ($K=100$) & - & - & \textbf{2.25} & -\\
        SCDM ($K=1000$) & - & - & - & 3.94\\
    \bottomrule
    \end{tabular}
    \end{center}
    \label{tab:fid}
\end{table}

The sufficient condition in Proposition~\ref{prop:better} is hard to verify in practice generally\footnote{A simple verifiable case is to fit a mixture of Gaussian (MoG) data by a single Gaussian (unconditional learning) or a MoG with ground-truth cluster indices (conditional learning).}. However, it does provide insights on when conditional learning is preferable. In fact, it implies that the conditional model gains more (i.e., $\min_E \gD (q(\vx|c) \| p_{\vtheta, E}(\vx))$ gets smaller for all $\vtheta$) as the conditional data distribution gets simpler. 
The condition is probably satisfied in practical conditional learning with class labels. In this sense, Proposition~\ref{prop:better} explains previous empirical findings.

\section{Self-Conditioned Diffusion Models}

Note that Proposition~\ref{prop:better} applies to all possible conditions, not limited to class labels, which inspires us to obtain meaningful conditions in an unsupervised manner to boost the unconditional generation results. The recent advances in self-supervised learning~\cite{he2020momentum,chen2020simple} show that one can learn predictive representations without labels, which serves as an ideal tool for obtaining meaningful conditions. 

Specifically, we propose a three-stage algorithm. Firstly, we train a feature extractor on the target dataset (without labels) in a self-supervised manner and extract features. Secondly, we run a clustering algorithm (e.g., $k$-means in our experiments) on these features and obtain the cluster indices for all data. Finally, we train a conditional diffusion model~\cite{nichol2021improved,dhariwal2021diffusion} by taking the cluster indices as conditions. We refer to our approach as \emph{self-conditioned diffusion models} (SCDM).

We mention that the high-level idea of using clustering indices from self-supervised learning coincides with prior work in GANs~\cite{armandpour2021partition,casanova2021instance,noroozi2020self}. This paper presents distinct contributions in the following aspects. First, prior work focuses on avoiding mode collapse while this paper is motivated by a different perspective with theoretical insights missing in the literature. Second, this paper is built upon SOTA diffusion models~\citep{chen2020improved,dhariwal2021diffusion} to explore the boundary of unconditional generative modeling. In fact, we obtain a record-breaking FID of 3.94 on ImageNet 64x64 without 
labels. See a direct comparison with prior work~\citep{casanova2021instance,noroozi2020self} in Table~\ref{tab:imagenet}.

\section{Experiment}

We evaluate SCDM on CIFAR10~\citep{krizhevsky2009learning}, CelebA 64x64~\citep{liu2015faceattributes}, LSUN Bedroom 64x64~\citep{yu15lsun} and ImageNet 64x64~\citep{deng2009imagenet}. By default, we use MoCo-v2~\cite{chen2020improved} on CIFAR10, CelebA 64x64 and LSUN Bedroom 64x64, and use MoCo-v3~\cite{chen2021empirical} on ImageNet 64x64, in the self-supervised learning stage. We use the FID score~\citep{heusel2017gans} to measure the sample quality. We use the same architecture for SCDM and its unconditional and conditional baselines. See more experimental details in Appendix~\ref{sec:setting}.

\subsection{Sample Quality}

Firstly, we compare our SCDM with the unconditional and conditional baselines. As shown in Table~\ref{tab:fid}, SCDM uniformly outperforms the unconditional model and slightly outperforms the conditional model on CIFAR10. On ImageNet 64x64, SCDM greatly improves the FID compared to the unconditional model. We provide generated samples in Figure~\ref{fig:samples}.

In Table~\ref{tab:imagenet}, we compare SCDM with other methods on ImageNet 64x64 in the unlabelled setting. SCDM significantly outperforms all prior methods and achieves a record-breaking FID of 3.94.

\begin{minipage}{\textwidth}
\begin{minipage}[b]{0.56\textwidth}
\centering
\includegraphics[width=\linewidth]{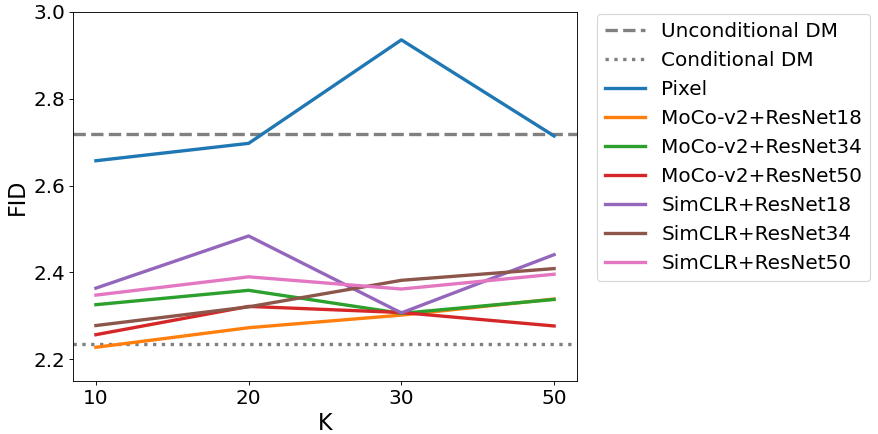}
\captionof{figure}{The effect of the self-supervised learning methods, and the backbones used in self-supervised learning.}
\label{fig:ab}
\end{minipage}
\hfill
\begin{minipage}[b]{0.42\textwidth}
\captionof{table}{ImageNet 64x64 results in the unlabelled setting. $^\dagger$Improved DDPM reports FID with 10K samples, and thereby we use reproduced results on 50K samples~\citep{bao2022analytic}.}
\begin{center}
\begin{tabular}{lc}
\toprule
Method & FID \\
\midrule
SLCGAN~\citep{noroozi2020self} & 19.2 \\
Unconditional BigGAN~\citep{casanova2021instance} & 16.9 \\
IC-GAN~\citep{casanova2021instance} & 9.2 \\
\midrule
Improved DDPM$^\dagger$~\citep{nichol2021improved} & 16.38 \\
Unconditional DM & 6.44 \\
SCDM (ours) & \textbf{3.94} \\
\bottomrule
\end{tabular}
\end{center}    
\label{tab:imagenet}
\end{minipage}
\end{minipage}

\begin{figure}[t]
\begin{center}
\subfloat[CIFAR10]{\includegraphics[width=0.235\columnwidth]{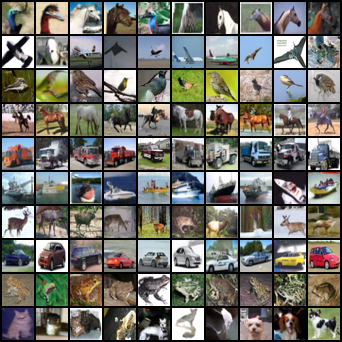}}
\hspace{.05cm}
\subfloat[CelebA 64x64]{\includegraphics[width=0.235\columnwidth]{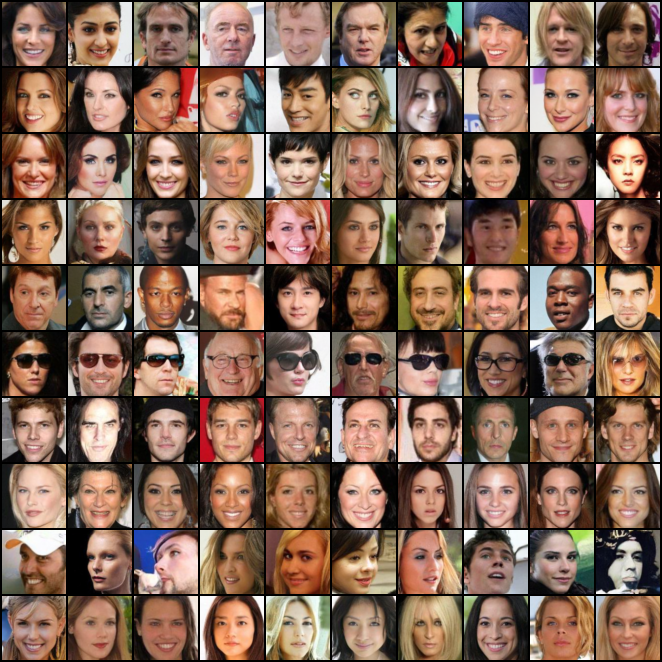}}  
\hspace{.05cm}
\subfloat[LSUN Bedroom 64x64]{\includegraphics[width=0.235\columnwidth]{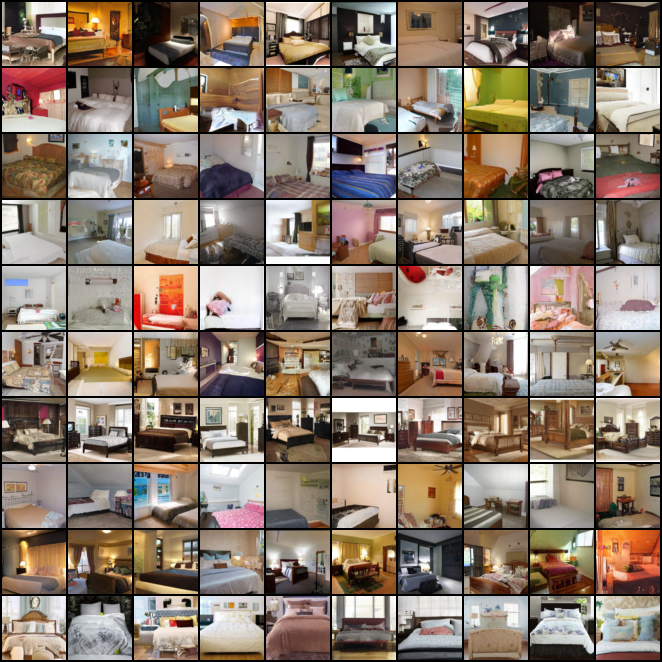}}
\hspace{.05cm}
\subfloat[ImageNet 64x64]{\includegraphics[width=0.235\columnwidth]{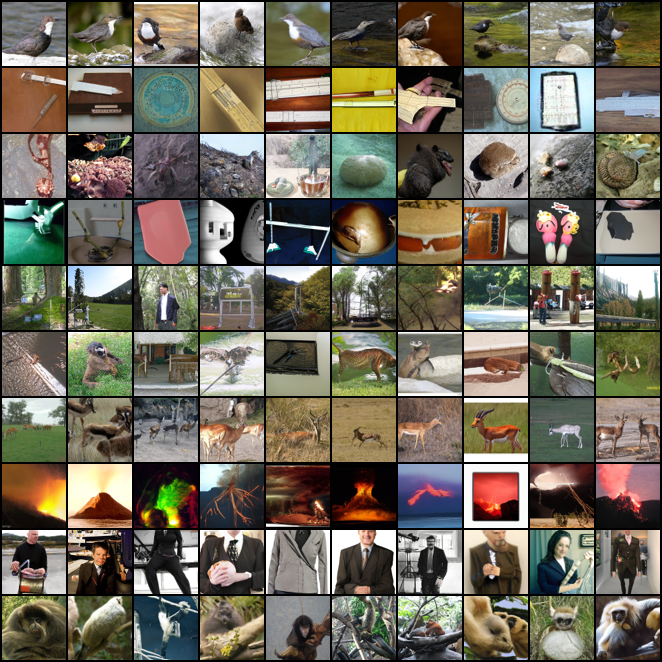}}
\vspace{-.1cm}
\caption{Generated samples of SCDM. Each column corresponds to a cluster. We use the model with the best FID.}
\label{fig:samples}
\end{center}
\vspace{-.8cm}
\end{figure}

\subsection{Ablation Study}

In this part, we study the effect of the self-supervised learning methods. We test MoCo-v2, as well as SimCLR~\citep{chen2020simple} with 3 backbones: ResNet-18, ResNet-34, and ResNet-50. We also perform $k$-means on image pixels directly to get cluster indices, and we call this method \textit{pixel}. As shown in Figure~\ref{fig:ab}, SimCLR performs similarly to MoCo-v2, and the choice of backbones does not affect the performance much. However, $k$-means on image pixels performs much worse than SimCLR and MoCo-v2. Indeed, as shown in Figure~\ref{fig:cmp}, we find objects of diverse classes appear in a single cluster for the pixel method, leading to a more complex distribution in a single cluster, which is more difficult to learn.

\begin{figure}[t]
\begin{center}
\subfloat[MoCo-v2 ($K=10$)]{\includegraphics[width=0.235\columnwidth]{imgs/cifar10_mocov2_resnet18_10.png}}
\hspace{.1cm}
\subfloat[SimCLR ($K=10$)]{\includegraphics[width=0.235\columnwidth]{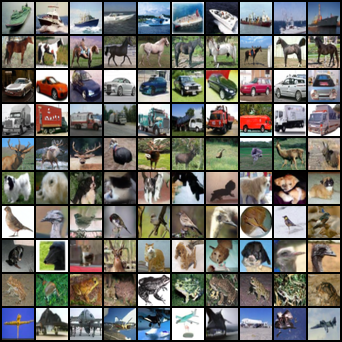}}
\hspace{.1cm}
\subfloat[Pixel ($K=10$)]{\includegraphics[width=0.235\columnwidth]{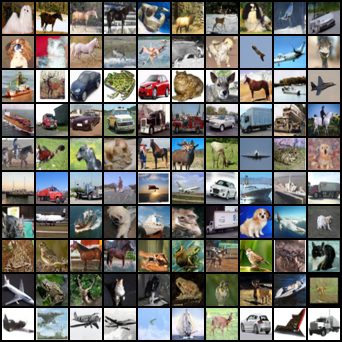}}
\caption{Generated samples on CIFAR10 with different clustering methods.}
\label{fig:cmp}
\end{center}
\vspace{-.8cm}
\end{figure}

\bibliographystyle{plain}
\bibliography{refs.bib}



\appendix

\section{Proof of Proposition~\ref{prop:better}}
\label{sec:proof}

We firstly present a lemma.
\begin{lemma}
\label{lem:better}
Suppose $\min\limits_{\vtheta \in \Theta} \E_{q(c)} \min\limits_{E \in \gE} \gD (q(\vx|c) \| p_{\vtheta, E}(\vx)) < \min\limits_{\vtheta \in \Theta, E \in \gE} \gD(q(\vx) \| p_{\vtheta, E}(\vx))$, the divergence $\gD$ is convex, and the embedding function space $\Phi$ includes all measurable functions. Then we have $\gD(q(\vx) \| p_{\vtheta_c^*, \vphi_c^*}(\vx)) < \gD(q(\vx) \| p_{\vtheta^*_u, E^*_u} (\vx))$.
\end{lemma}

\begin{proof}
According to the convexity of $\gD$, we have 
\begin{align}
\label{eq:better1}
    \gD(q(\vx) \| p_{\vtheta_c^*, \vphi_c^*}(\vx)) = \gD(\E_{q(c)}q(\vx|c) \| \E_{q(c)} p_{\vtheta_c^*, \vphi_c^*}(\vx|c)) \leq \E_{q(c)} \gD(q(\vx|c) \| p_{\vtheta_c^*, \vphi_c^*}(\vx|c))
\end{align}

According to the definition of $\vtheta_c^*, \vphi_c^*$, we have
\begin{align}
    & \E_{q(c)} \gD(q(\vx | c) \| p_{\vtheta^*_c, \vphi^*_c}(\vx|c)) = \min_{\vtheta, \vphi} \E_{q(c)} \gD(q(\vx|c) \| p_{\vtheta, \vphi}(\vx|c)) \nonumber \\
    = & \min_{\vtheta, \vphi} \E_{q(c)} \gD(q(\vx|c) \| p_{\vtheta, \vphi(c)}(\vx)) = \min_{\vtheta}  \E_{q(c)} \min_{\vphi(c)} \gD(q(\vx|c) \| p_{\vtheta, \vphi(c)}(\vx)) \nonumber \\
    = & \min_{\vtheta}  \E_{q(c)} \min_E \gD(q(\vx|c) \| p_{\vtheta, E}(\vx)). \label{eq:better2}
\end{align}

Combining Eq.~{(\ref{eq:better1})}, Eq.~{(\ref{eq:better2})}, and the assumption, we have
\begin{align*}
    & \gD(q(\vx) \| p_{\vtheta_c^*, \vphi_c^*}(\vx)) \leq \min_{\vtheta}  \E_{q(c)} \min_E \gD(q(\vx|c) \| p_{\vtheta, E}(\vx)) \\
    < & \min\limits_{\vtheta, E} \gD(q(\vx) \| p_{\vtheta, E}(\vx)) = \gD(q(\vx) \| p_{\vtheta_u^*, E_u^*}(\vx)).
\end{align*}

\end{proof}

Then we present proof of Proposition~\ref{prop:better}.
\begin{proof}
Since $\forall \vtheta, c, \min\limits_E \gD (q(\vx|c) \| p_{\vtheta, E}(\vx)) < \min\limits_E \gD(q(\vx) \| p_{\vtheta, E}(\vx))$,
we have
\begin{align*}
    \min\limits_\vtheta \E_{q(c)} \min\limits_E \gD (q(\vx|c) \| p_{\vtheta, E}(\vx)) < \min\limits_{\vtheta, E} \gD(q(\vx) \| p_{\vtheta, E}(\vx)).
\end{align*}

According to Lemma~\ref{lem:better}, we have $\gD(q(\vx) \| p_{\vtheta_c^*, \vphi_c^*}(\vx)) < \gD(q(\vx) \| p_{\vtheta^*_u, E^*_u} (\vx))$.
\end{proof}

\section{Experimental Details}
\label{sec:setting}
In the self-supervised learning stage, we use ResNet18 on CIFAR10, and ResNet50 on CelebA and LSUN Bedroom. We train 1600, 800 and 200 epochs on CIFAR10, CelebA and LSUN Bedroom respectively. We resize images to 32x32, and use a batch size of 512. We use the SGD optimizer with a learning rate of 0.06, a momentum of 0.9 and a weight decay of 5e-4. The queue size of MoCo is 12800, and the momentum of MoCo is 0.999. As for ImageNet, we use pretrained ViT-Base provided in \url{https://github.com/facebookresearch/moco-v3}.

We provide training and sampling settings of diffusion models in Table~\ref{tab:setting}. We evaluate FID every 50K training steps and report the best one.

\begin{table}[H]
\begin{center}
\begin{tabular}{lcccc}
\toprule
    Dataset & CIFAR10 & CelebA 64x64 & LSUN Bedroom 64x64 & ImageNet 64x64 \\
    \midrule
    Architecture & IDDPM~\citep{nichol2021improved} & IDDPM & IDDPM & ADM~\citep{dhariwal2021diffusion} \\
    Noise schedule & VP SDE~\citep{song2020score} & VP SDE & VP SDE & VP SDE \\
    Batch size & 128 & 128 & 128 & 2048 \\
    Training steps & 1M & 1M & 1M & 550K \\
    Optimizer & Adam~\citep{kingma2014adam} & Adam & Adam & AdamW~\citep{loshchilov2017decoupled} \\
    Learning rate & 1e-4 & 1e-4 & 1e-4 & 3e-4 \\
    Sampler & EM & EM & EM & DPM-Solver~\citep{lu2022dpm} \\
    Sampling steps & 1K & 1K & 1K & 50 \\
    \bottomrule
\end{tabular}
\end{center}
\caption{The experimental setting of diffusion models. EM represents the Euler-Maruyama sampler.}
\label{tab:setting}
\end{table}

\end{document}